\documentclass{article}

\usepackage{xspace} 
\usepackage{microtype}
\usepackage{graphicx}
\usepackage{subcaption}
\usepackage{booktabs} 
\usepackage{hyperref}
\usepackage{enumitem}
\usepackage{cuted}
\usepackage{caption}
\usepackage{tabularx}
\usepackage{multirow}
\usepackage{algpseudocode}
\usepackage{wrapfig}

\usepackage{enumitem}
\newcommand{\largebold}[1]{\scalebox{1.2}{\textbf{#1}}}

\usepackage[preprint]{icml2026}

\usepackage{amsmath}
\usepackage{amssymb}
\usepackage{mathtools}
\usepackage{amsthm}
\usepackage{sidecap}

\usepackage{xstring}
\usepackage{xcolor}
\definecolor{mydarkblue}{rgb}{0,0.08,0.45}
\definecolor{airforceblue}{rgb}{0.36, 0.54, 0.66}
\hypersetup{
    colorlinks=true,
    linkcolor=airforceblue,
    filecolor=airforceblue,      
    urlcolor=airforceblue,
    citecolor=airforceblue,
    }
\definecolor{segray}{gray}{0.45}

\newcommand{\se}[1]{\textcolor{segray}{\tiny\,\ensuremath{\pm}\,{#1}}}
\newcommand{\semth}[1]{\textcolor{segray}{\scriptstyle\,\pm\,#1}}

\usepackage[compact]{titlesec}
\titlespacing{\section}{0pt}{1.2ex}{0.6ex}
\titlespacing{\subsection}{0pt}{1.0ex}{0.5ex}

\usepackage[capitalize,noabbrev]{cleveref}

\newcommand{\App}{App.\@ }

\theoremstyle{plain}

\theoremstyle{definition}

\theoremstyle{remark}

\newtheorem{claim}{Claim}

\usepackage[textsize=tiny]{todonotes}

\newcommand{\methodname}{\emph{VISTA}\xspace}
\newcommand{\methodfull}{\textbf{V}alidation-Informed \textbf{S}elf-\textbf{T}rajectory \textbf{A}daptation\xspace} 

\icmltitlerunning{\methodname: Validation-Informed Trajectory Adaptation via Self-Distillation}

\begin{document}

\twocolumn[
  \icmltitle{\methodname: Validation-Informed Trajectory Adaptation via Self-Distillation}

  \icmlsetsymbol{equal}{*}

  \begin{icmlauthorlist}
    \icmlauthor{Eli Corn}{yyy} 
    \icmlauthor{Daphna Weinshall}{yyy} 
  \end{icmlauthorlist}

  \icmlaffiliation{yyy}{Department of Computer Science, Hebrew University of Jerusalem, Jerusalem, Israel} 

  \icmlcorrespondingauthor{Eli Corn}{eli.corn@mail.huji.ac.il} 
  \icmlcorrespondingauthor{Daphna Weinshall}{daphna@mail.huji.ac.il} 
  \icmlkeywords{Machine Learning, ICML}
  \vskip 0.3in
]


\printAffiliationsAndNotice{}  

\begin{abstract}
\methodfull
Deep learning models may converge to suboptimal solutions despite strong validation accuracy, masking an optimization failure we term \emph{Trajectory Deviation}. This is because as training proceeds, models can abandon high-generalization states for specific data sub-populations, thus discarding previously learned latent features without triggering classical overfitting signals. To address this problem we introduce \methodname, an online self-distillation framework that enforces consistency along the optimization trajectory. Using a validation-informed Marginal Coverage score, \methodname identifies expert anchors, which are earlier model states that retain specialized competence over distinct data regions. A coverage-weighted ensemble of these anchors is integrated online during the training of a \emph{single model}, regularizing the loss landscape and preserving mastered knowledge. When evaluated across multiple benchmarks, \methodname demonstrates improved robustness and generalization over standard training and prior self-distillation methods, while a lightweight implementation reduces storage overhead by 90\% without performance loss.


\end{abstract}

\section{Introduction}


In the standard deep learning paradigm, optimization success is typically assessed by the monotonic improvement of global validation metrics. However, recent analysis of training dynamics \citep{stern2025local} reveals that this macroscopic trend frequently masks a critical failure mode we call \emph{Trajectory Deviation}: a process where the model identifies high-generalization features for specific latent structures during mid-training, only to subsequently abandon these regions in favor of broader global optimization. This localized erosion is particularly elusive because it persists undetected by aggregate performance tracking; while total accuracy continues to rise, the model functionally discards previously mastered features, resulting in a converged state that is structurally suboptimal. Crucially, this optimization instability is not adequately addressed by standard regularization, early stopping, or traditional self-distillation.


\begin{center}
    \centering
    \includegraphics[width=.9\columnwidth]{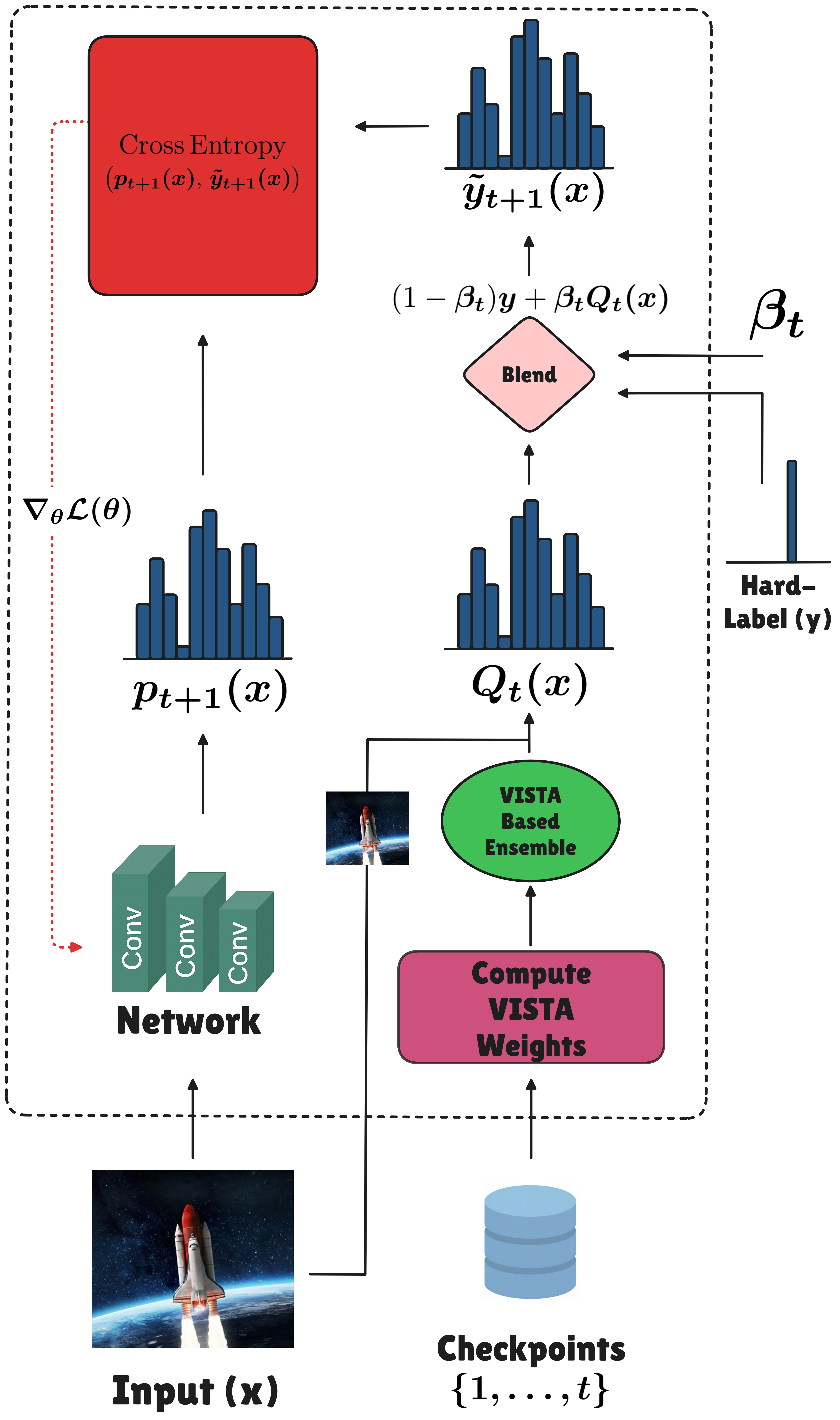}
    
    \captionof{figure}{
        Overview of the \methodname self-distillation flow at training epoch $t{+}1$: predictions from epochs ${1..t}$ are combined with the one-hot label to form a coverage-weighted teacher target.
    }
    \label{fig:cafe_flow_2d}
\end{center}

\begin{figure*}[t!] 
\begin{center}
\includegraphics[width=0.95\textwidth]{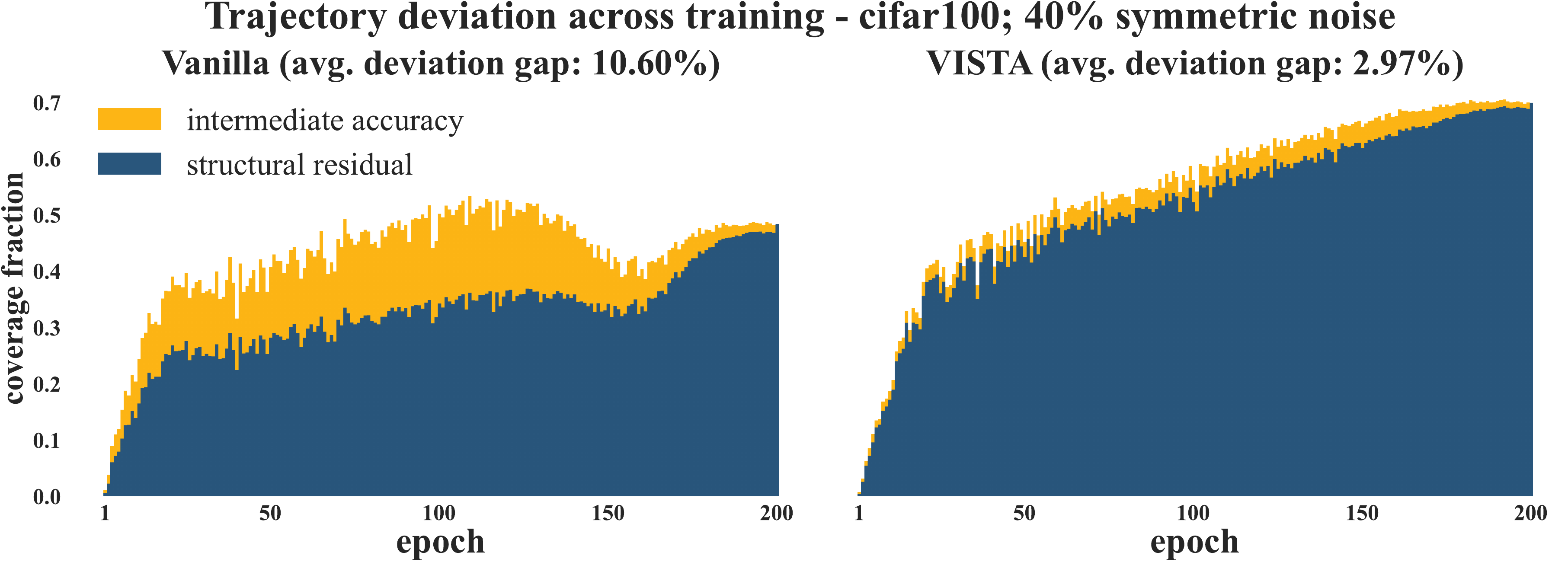} 
\end{center}
\vspace{-6pt} 
\raggedright{\hspace{0.17\textwidth}\small{(a) \hspace{0.43\textwidth} (b)}}
\caption{Trajectory Deviation and Structural Consistency. Yellow bars denote \emph{Intermediate Accuracy} at each checkpoint, while blue bars indicate the \emph{Structural Residual}: the subset of those specific predictions that remain correct in the converged model. The vertical gap represents the \emph{Deviation Gap} - mastered knowledge that was later discarded. (a) Standard CE baseline; (b) \methodname.  
}
\vspace{-0.2cm}
\label{fig:forget_ce_vs_cafe}
\end{figure*}




We illustrate the impact of \emph{Trajectory Deviation} in Figure~\ref{fig:forget_ce_vs_cafe} using two overlapping bars per checkpoint. The yellow bars represent a model's checkpoint accuracy on a validation set $\mathcal{V}$, which we term the \textbf{Intermediate Accuracy}. In contrast, the blue bars represent the \textbf{Structural Residual}: the subset of those specific samples that the model correctly identified at that epoch and successfully retained until the end of training. Formally, let $C_e=\{i\in \mathcal{V}:\text{correct at epoch }e\}$ and $C_E=\{i\in \mathcal{V}:\text{correct at the final epoch }E\}$. The vertical gap between these bars, which we call the \textbf{DeviationGap}($e$), is given by:
\begin{equation}
\begin{aligned}
    &\mathrm{Deviation Gap}(e) := \\
    & \quad \underbrace{\frac{|C_e|}{|\mathcal{V}|}}_{\text{Intermediate Accuracy}} - \underbrace{\frac{|C_e \cap C_E|}{|\mathcal{V}|}}_{\text{Structural Residual}} = \frac{|C_e \setminus C_E|}{|\mathcal{V}|}
\end{aligned}
\end{equation}
This gap quantifies the optimization trajectory’s drift away from high-fidelity latent regions. As shown in Figure~\ref{fig:forget_ce_vs_cafe}a, standard training exhibits large gaps, indicating that the final model fails on samples it previously classified correctly. 

Traditionally, to mitigate such information loss, researchers have proposed variants of \emph{temporal ensembling} \citep{laine2017temporal}. These methods update the current loss using an average of past predictions to smooth the training signal and enhance generalization. As an alternative, several approaches construct ensembles of intermediate checkpoints to achieve a similar effect \citep{toosifar2025subnetwork}. Building on the emerging notion of localized optimization failure, the specialized \emph{Knowledge Fusion (KF)} method addresses this by selecting a subset of checkpoints based on a validation-driven score and combining them into a weighted ensemble for inference. While effective, this approach faces notable practical limitations: it often requires massive storage for checkpoints and increase inference complexity.

In Section~\ref{sec:method}, we propose a different online approach which follows the paradigm of temporal ensembling while specifically targeting the manifestations of \emph{Trajectory Deviation}. Specifically, in \methodname (illustrated in Figure~\ref{fig:cafe_flow_2d}), we adopt a coverage-based notion to steer training online, introducing a \emph{Marginal Coverage} score that identifies past model states acting as ``expert anchors'' for specific data regions. Although this training-time score differs from the deviation gap, which refers to the final model, Figure~\ref{fig:forget_ce_vs_cafe} illustrates its impact: with \methodname, accuracy is elevated across epochs, the double-descent dynamic vanishes, and the deviation gap is markedly decreased. 

\methodname bridges two previously independent ideas to enforce structural consistency. On one hand, self-distillation stabilizes training by reusing a model’s own predictions; on the other hand, validation performance can reveal which patterns are at risk of being discarded. By combining these notions, \methodname transforms self-distillation into a targeted, online mechanism for rolling mastered knowledge forward. This integration turns insights about optimization paths into a practical, single-model procedure that preserves knowledge without sacrificing efficiency. In Section~\ref{sec:results}, we compare \methodname with a broad suite of alternative baselines, showing its overall superiority in providing better accuracy while maintaining equal or lower complexity.

Summary of main contributions:
\begin{itemize}[leftmargin=*,itemsep=0.2em, topsep=0.em, parsep=0pt, partopsep=0pt]
    \item We introduce the \emph{Marginal Coverage} score to identify and prioritize expert states along the model's self-trajectory.
    \item We propose \textbf{\methodname}, an online self-distillation framework that uses these scores to enforce structural consistency during training.
    \item We demonstrate superior performance and robustness over SOTA baselines across multiple architectures with single-model inference.
    \item We provide a lightweight variant that achieves a 90\% reduction in storage overhead while preserving peak performance.
\end{itemize}

\section{Related work}
\label{sec:related_work}

\paragraph{Leveraging past training information}
Several methods smooth the training process by leveraging information from previous epochs instead of relying only the current loss. Temporal ensembling \citep{laine2017temporal} averages predictions across epochs to improve semi-supervised learning. In noisy-label settings, approaches such as MentorNet \citep{jiang2018mentornet} and Meta-Weight-Net \citep{shu2019meta} down-weight unreliable samples using historical loss statistics. More recently, EMA-based methods have been shown to improve robustness to noisy labels and generalization \citep{morales2024exponential}. Additionally, checkpoint ensembles have been proposed to aggregate multiple training snapshots, enhancing performance \citep{toosifar2025subnetwork}. 


\paragraph{Trajectory Deviation and Knowledge Fusion}
Recent work shows that deep networks can lose previously acquired structure late in training, as in trajectory deviation, where models abandon high-generalization states for specific data sub-populations. Knowledge Fusion (KF)~\cite{stern2025local} addresses this via a weighted ensemble of validation-selected checkpoints, but its post-hoc design, checkpoint storage, hyperparameter sweeps, and ensemble inference overhead limit its scalability.

\paragraph{Self-distillation baselines and scope}
We benchmark \methodname against three main families of self-distillation techniques that rely on \emph{soft targets} but differ in how the “teacher” is formed.  
(1) \textbf{Label-smoothing / teacher-free methods:} Teacher-Free KD (TF-KD)~\citep{yuan2020tfd} and Zipf-scheduled label smoothing (Zipf’s LS)~\citep{zipf_table2_proxy}.  
(2) \textbf{Current-epoch teachers:} methods that form teachers from intermediate predictions within the same epoch, including BYOT~\citep{zhang2019byot}, DLB~\citep{shen2024dlb}, CS-KD~\citep{yun2020cskd}, DDGSD~\citep{xu2019ddgsd}, DKS~\citep{sun2019dks}, FRSKD~\citep{ji2021frskd}, and Self-Distillation with Dropout~\citep{zhu2022sddropout}.  
(3) \textbf{Past-epoch teachers:} methods that aggregate predictions across training time, such as PS-KD~\citep{kim2020pskd}, SAT~\citep{huang2020sat}, EWR-KD~\citep{xia2021ewrkd}, TSD~\citep{liu2024tsd}, and KF.  
For completeness we also report Born-Again Networks (BAN)~\citep{furlanello2018ban}, a widely used multi-round self-distillation baseline. 


\paragraph{How the baseline methods construct soft targets}

\emph{(1) Label-smoothing / teacher-free.} TF-KD reframes distillation as label smoothing without a teacher, while Zipf’s LS shapes non-target probabilities via a Zipf distribution at negligible cost.
\emph{(2) Current-epoch teachers.} BYOT transfers from deeper to shallower heads; DLB regularizes with last mini-batch predictions; CS-KD enforces class consistency; DDGSD aligns augmented views; DKS adds auxiliary branches with pairwise consistency; FRSKD refines features/logits through an auxiliary branch; and SD with Dropout ensembles stochastic subnetworks.
\emph{(3) Past-epoch teachers.} SAT uses momentum-averaged targets with reweighting; EWR-KD adds uncertainty-aware reweighting; TSD builds tolerant targets from stored predictions; and KF ensembles checkpoints via validation-driven forgetting scores. BAN retrains successive student generations. 
Finally, PS-KD, the most competitive baseline (see Table~\ref{tab:cifar100_full_metrics}), uniformly aggregates prior predictions. In contrast, \methodname is guided by \emph{validation coverage of expert anchors}, identifying specific data regions at risk of trajectory deviation. This targeted weighting generates distinct soft targets for online self-distillation within a single training run.

\section{Method}
\label{sec:method}

\methodname, described in Sections~\ref{cafe-basic}-\ref{sec:loss}, trains a \emph{single} network in a single pass. It begins by asking the following question at each epoch: \textit{Which past states act as expert anchors for regions of the feature space that the current model is starting to deviate from?} 
To answer this question, we developed a notion of coverage with respect to a small held-out 
validation set, which is used to estimate test-set coverage\footnote{\App~\ref{app:val_as_surrogate} provides empirical evidence 
showing that validation accuracy closely tracks test accuracy even under label noise
.}. 
If a checkpoint correctly classifies validation examples that subsequent checkpoints err on, we say that it provides \emph{marginal coverage}. This \emph{coverage} score determines the influence assigned to the respective checkpoint, when forming the soft targets that guide subsequent training.

A lightweight variant that significantly reduces space and time complexity is described in Section~\ref{subsec:light-cafe}. The performance difference between the two variants is examined in the ablation study (Table~\ref{tab:fat_vs_thresh}). 
Section~\ref{sec:asys} summarizes the theoretical justification and complexity analysis, while \App~\ref{app:theo_and_comlex_asys} contains the complete derivations and proofs.

\subsection{Basic \methodname method}
\label{cafe-basic}

Algorithm \ref{alg:fat_coverage_distill} presents the procedure used by \methodname to generate blended targets after completing epoch $t$. These targets are subsequently used for optimization in epoch $t+1$. Each step of this per-epoch computation is described in detail in the remainder of this section (see Figure~\ref{fig:cafe_flow_2d}).

\begin{algorithm}[htb]
\caption{\methodname: epoch $t$} 
\label{alg:fat_coverage_distill}

\textbf{Input}:  
Validation set $\mathcal{V}$; 
coverage subsets $\{C_s\}_{s=1}^{t}$; 
validation accuracies $\{a_s\}_{s=1}^{t}$; 
checkpoint predictions $\{p_s(x)\}_{s=1}^{t}$; 
scheduler $\beta_t \in [0,1]$.

\textbf{Output}: Blended target for next epoch optimization.

\begin{algorithmic}[1]
\State \textbf{Order checkpoints:} $S_t \gets \{1,\dots,t\}$ sorted by $a_s$ (high $\to$ low) 
\label{step1}
\State \textbf{Marginal coverage sweep:} $U \gets \emptyset$; \textbf{for} $s \in S_t$ \textbf{do} $\Delta_s \gets |C_s \setminus U|$; \; $U \gets U \cup C_s$ 
\label{step2}
\State \textbf{Normalize coverage gains:} $\widehat{\Delta}_s \gets \Delta_s \Big/ \sum_{j \in S_t} \Delta_j$ 
\label{step3}
\State \textbf{Construct coverage-weighted teacher:} $Q_t(x) \gets \sum_{s \in S_t} \widehat{\Delta}_s \, p_s(x)$ 
\label{step4}
\State \textbf{Form blended target for next epoch:} $\tilde y_{t+1}(x) \gets (1-\beta_t)\,\text{one-hot}(y) + \beta_t\,Q_t(x)$ 
\label{step5}
\State \textbf{Return}  $\tilde y_{t+1}(x)$
\end{algorithmic}
\end{algorithm}

\subsection*{Step~\ref{step1}: Order checkpoints}
\vspace{-0.2cm}
All checkpoints in the range $s\in[1\ldots t]$ are sorted based on their accuracy on the validation set $\mathcal{V}$, providing the sorted list $S_t=\{s_j\}_{j=1}^t$. We then sweep through this ordered list from beginning (most accurate) to end (least accurate).


\subsection*{Step~\ref{step2}: Marginal coverage sweep}
\vspace{-0.2cm}
While sweeping through list $S_t$, we assign to each checkpoint its marginal coverage of the validation set. More specifically, the first checkpoint $s_1$ (with the highest accuracy) is assigned $\vert C_{s_1}\vert$ - the number of points that it correctly classifies in the validation set. The second checkpoint $s_2$ (with the second highest accuracy) is assigned $\vert C_{s_2} \setminus C_{s_1}\vert$ - the number of points in the validation set that it correctly classifies, and that checkpoint $s_1$ does \textbf{not} classify correctly. This continues until all the checkpoints are assigned a non-negative value, according to the following formulae:
\begin{equation}
\label{eq:weights}
\Delta_{s_i} \;=\; \bigl|\, C_{s_i} \setminus \bigcup_{j=1}^{i-1} C_{s_j} \,\bigr|.
\end{equation}

Figure~\ref{fig:marginal_coverage} illustrates how marginal coverage is computed after sorting checkpoints by validation accuracy. We relabel the sorted checkpoints as $s_1$, $s_2$, and $s_3$, so that $|C_{s_1}| \ge |C_{s_2}| \ge |C_{s_3}|$. The left-hand panel shows the validation examples each checkpoint correctly classifies. Since $C_{s_1}$ and $C_{s_2}$ do not overlap ($C_{s_1} \cap C_{s_2} = \emptyset$), their marginal coverages equal their total sizes: $\Delta_{s_1} = |C_{s_1}|$ and $\Delta_{s_2} = |C_{s_2}|$. Checkpoint $s_3$ overlaps with earlier ones, so only its unique contribution $C_{s_3} \setminus (C_{s_1} \cup C_{s_2})$ counts, giving $\Delta_{s_3} < |C_{s_3}|$. This highlights the dependence on ordering, ensuring that weaker checkpoints are credited only for patterns not already captured by stronger ones.


\begin{figure}[h]
  \centering
  \includegraphics[width=0.75\linewidth]{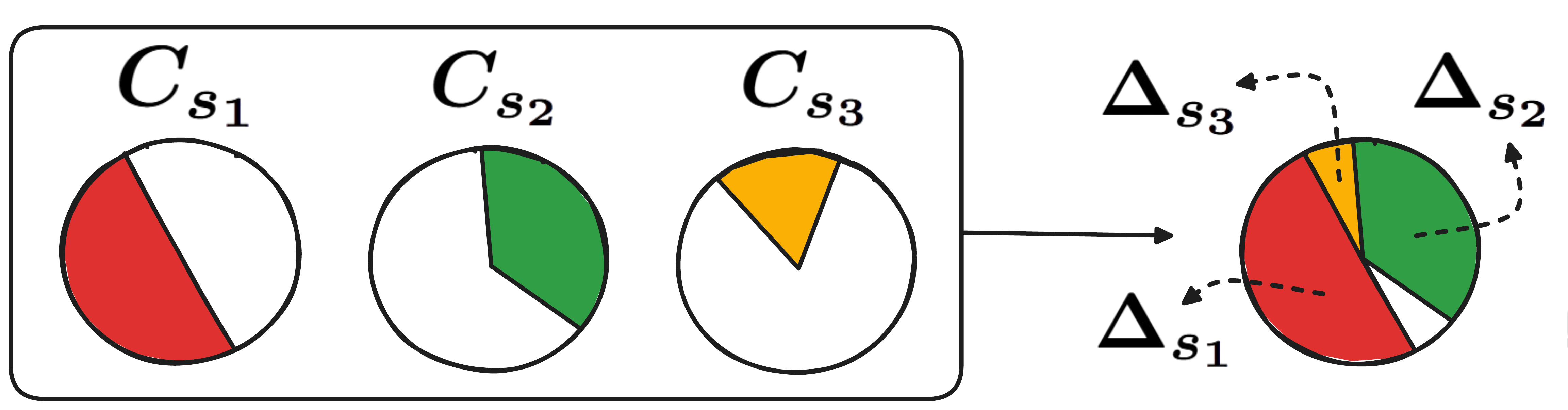}
  \caption{
  Illustration of marginal coverage: checkpoints are ranked by validation accuracy (left), 
  then their unique contributions are computed (right) to produce $\Delta_s$ values, to be used for weighting. 
  }
  \label{fig:marginal_coverage}
\end{figure}

\subsection*{Steps~\ref{step3} and~\ref{step4}: coverage-weighted teacher}
\vspace{-0.2cm}
Next, we define the \emph{teacher distribution} $Q_t(x)$ by averaging checkpoint predictions with weights proportional to their \emph{marginal validation coverage}; this distribution is blended later with the hard label in the loss:
\begin{equation}
\label{eq:teacher}
Q_t(x)\;=\;\sum_{s\in S_t}\widehat{\Delta}_s\,p_s(x),\qquad 
\widehat{\Delta}_s \;=\; \frac{\Delta_s}{\sum_{j\in S_t}\Delta_j},
\end{equation}

where $S_t$ is the set of checkpoints considered up to epoch $t$
, $p_s(x)$ their post-softmax class probabilities (prediction vector), and $\Delta_s$ the marginal coverage counts from Step~\ref{step2}.

\subsection{\methodname loss at epoch \ensuremath{t+1}}
\label{sec:loss}

We obtain the \methodname loss by blending the teacher with the hard label using a \emph{parameter-free} exponential schedule:
\[
u_t\;=\;\frac{t-1}{E-1}\in[0,1],\qquad
\beta_t\;=\;\frac{1-\exp(-2\,u_t)}{1-\exp(-2)}.
\]
Here, $E$ denotes the total number of training epochs, so that $u_t$ linearly interpolates between $0$ and $1$ over the course of training.

Gradually increasing reliance on the soft teacher has been shown to improve stability and generalization in self-distillation, 
we therefore choose an exponential rise for $\beta_t$, mirroring the typical exponential decay of train error \citep{hestness2017deep}, so that reliance on $Q_t(x)$ grows smoothly as the model stabilizes.  The constant “2” is there to ensure a fast-then-plateau curve.

Given $Q_{t}(x)$ and $\beta_{t}$, the per-sample target $\tilde y_{t+1}(x)$ and loss $\mathcal{L}_{t+1}$ can now be derived:
\begin{align*}
    \tilde y_{t+1}(x) &= (1-\beta_{t})\,\text{one-hot}(y) + \beta_{t}\,Q_{t}(x), \\
    \mathcal{L}_{t+1} &= \text{CE}\big(p_{t+1}(x), \tilde y_{t+1}(x)\big).
\end{align*}

Above $\mathrm{CE}$ denotes the cross-entropy, namely, the $KL$ divergence to $\tilde y_{t+1}(x)$ up to a constant.
This simple convex combination lets the hard label anchor the semantics, while the coverage-weighted teacher
injects structure that preserves information otherwise prone to trajectory deviation.

\begin{table*}[t]
\centering
\scriptsize
\caption{\textbf{CIFAR-100N and CIFAR100.} Mean test accuracy (\%, error over 3 random seeds) on image classification datasets using a ResNet-18 backbone. The best performer in each column is highlighted in bold. The last row shows the improvement of the best performer over Vanilla ERM with early stopping.} 
\label{table:cifar100_combined}
\resizebox{\textwidth}{!}{
\begin{tabular}{l||l||l|l|l|l||l|l}
\toprule
\multicolumn{1}{l||}{\textbf{Method / Dataset}}
  & \multicolumn{1}{c||}{\textbf{CIFAR-100N}}
  & \multicolumn{4}{c||}{\textbf{CIFAR100 (Symmetric Noise)}} 
  & \multicolumn{2}{c}{\textbf{CIFAR100 (Asymmetric Noise)}} \\
\multicolumn{1}{r||}{noise level}
  & 40\% human
  & Clean & 20\% & 40\% & 60\%
  & 20\% & 40\% \\
\midrule
\emph{Vanilla + Early Stopping}
  & $54.53 \semth{.44}$
  & $78.66 \semth{.11}$ & $64.98 \semth{.11}$ & $59.37 \semth{.36}$ & $50.92 \semth{.48}$
  & $66.61 \semth{.12}$ & $49.57 \semth{.10}$ \\
\midrule
\emph{SAT}
  & $53.18 \semth{.43}$
  & $77.84 \semth{.01}$ & $71.97 \semth{.10}$ & $66.67 \semth{.09}$ & $55.79 \semth{.91}$
  & $74.43 \semth{.07}$ & $63.26 \semth{.13}$ \\
\emph{KF Ensemble}
  & \textemdash
  & $79.13 \semth{.14}$ & $72.80 \semth{.10}$ & $67.00 \semth{.10}$ & \textemdash
  & $74.20 \semth{.10}$ & $62.10 \semth{.50}$ \\
\methodname
  & $\mathbf{63.92} \semth{.08}$
  & $\mathbf{80.25} \semth{.05}$ & $\mathbf{74.09} \semth{.01}$ & $\mathbf{69.75} \semth{.14}$ & $\mathbf{59.42} \semth{.26}$
  & $\mathbf{75.61} \semth{.18}$ & $\mathbf{67.17} \semth{.10}$ \\
\midrule
\emph{Improvement (over Vanilla)}
  & $\largebold{9.39} \semth{.45}$
  & $\largebold{1.59} \semth{.12}$ & $\largebold{9.11} \semth{.11}$ & $\largebold{10.38} \semth{.39}$ & $\largebold{8.5} \semth{.54}$
  & $\largebold{9.0} \semth{.22}$ & $\largebold{17.6} \semth{.14}$ \\
\bottomrule
\end{tabular}
}
\end{table*}

\begin{table*}[t]
\centering
\scriptsize
\caption{\textbf{TinyImageNet.} Mean test accuracy comparing our method and baselines. Mean test accuracy, see caption of Table~\ref{table:cifar100_combined}.}
\label{table:tinyimagenet_combined}
\resizebox{\textwidth}{!}{
\begin{tabular}{l||l|l|l|l||l|l}
\toprule
\multicolumn{1}{l||}{\textbf{Method / Dataset}}
  & \multicolumn{4}{c||}{\textbf{TinyImageNet (Symmetric Noise)}} 
  & \multicolumn{2}{c}{\textbf{TinyImageNet (Asymmetric Noise)}} \\
\multicolumn{1}{r||}{noise level}
  & Clean & 20\% & 40\% & 60\%
  & 20\% & 40\% \\
\midrule
\emph{Vanilla + Early Stopping}
  & $65.37 \semth{.01}$ & $56.32 \semth{.22}$ & $49.62 \semth{.23}$ & $40.02 \semth{.38}$
  & $58.12 \semth{.18}$ & $43.42 \semth{.17}$ \\
\midrule
\emph{SAT}
  & $63.10 \semth{.20}$ & $60.59 \semth{.17}$ & $54.01 \semth{.20}$ & $39.92 \semth{.23}$
  & $62.45 \semth{.02}$ & $52.73 \semth{.10}$ \\
\emph{KF Ensemble}
  & $68.50 \semth{.36}$ & $62.80 \semth{.20}$ & $57.00 \semth{.50}$ & \textemdash
  & \textemdash &  \textemdash \\
\methodname
  & $\mathbf{68.91} \semth{.01}$ & $\mathbf{63.20} \semth{.03}$ & $\mathbf{58.20} \semth{.16}$ & $\mathbf{47.12} \semth{.20}$
  & $\mathbf{64.71} \semth{.03}$ & $\mathbf{53.69} \semth{.04}$ \\
\midrule
\emph{Improvement (over Vanilla)}
  & $\largebold{3.54} \semth{.01}$ & $\largebold{6.88} \semth{.22}$ & $\largebold{8.58} \semth{.28}$ & $\largebold{7.10} \semth{.43}$
  & $\largebold{6.59} \semth{.18}$ & $\largebold{10.27} \semth{.17}$ \\
\bottomrule
\end{tabular}
}
\end{table*}

\subsection{Light \methodname}
\label{subsec:light-cafe}

At epoch $t$, \methodname\ computes the \emph{marginal validation coverage} 
$\Delta_s$ of all past checkpoints and forms a coverage-weighted teacher 
as in~(\ref{eq:teacher}).  
We introduce a lighter, more efficient variant - \emph{Light} \methodname - that 
uses a tolerance parameter $\tau \ge 0$ to prune checkpoints with negligible 
marginal coverage.  
When $\tau=0$, \textbf{\emph{Light} \methodname is provably equivalent to Basic \methodname}, yet it dramatically reduces storage and computation requirements.  
Specifically, we leverage the following property:

\begin{claim}[Restated]
\label{claim:1}
Let $\mathcal{B}_t = \{s \in S_t \mid \Delta_s^t > 0\}$ be the set of checkpoints with non-zero marginal coverage at epoch $t$. For any $T > t$, $\mathcal{B}_T \cap \{S_t \setminus \mathcal{B}_t\} = \emptyset$.
\end{claim}
Intuitively, since a checkpoint's redundancy is measured against a union of sets that only grows over time, any checkpoint discarded today will never provide new information in the future. This allows \methodname\ to permanently prune the model pool (see \App~\ref{app:proof_tau_equivalence} for the full proof).

The complete pseudocode can be found in \App~\ref{app:light-cafe-alg}. The resulting 
storage effects are detailed in 
Section~\ref{sec:storage_empirical}.

\paragraph{Difference from Basic \methodname:} Consider set $S_t$ in (\ref{eq:teacher}), and let $\mathcal{A}_t\subseteq   S_t$ denote the set of \emph{contributing checkpoints}, which are all checkpoints for which $\Delta_s~>~\tau$ for some $\tau\geq 0$. In the lightweight version of \methodname, teacher $Q_t(x)$ is redefined in Step
~\ref{step4} as follows:
\begin{equation}
\label{eq:teacher-tau}
Q_t(x)\;=\;\sum_{s\in \mathcal{A}_t}\widehat{\Delta}_s\,p_s(x),\qquad 
\widehat{\Delta}_s \;=\; \frac{\Delta_s}{\sum_{j\in \mathcal{A}_t}\Delta_j}.
\end{equation}
The empirical findings in Section~\ref{sec:storage_empirical} demonstrate that a conservative choice of $\tau$ can shrink the size of the contributing checkpoints set $\vert\mathcal{A}_t\vert$ by orders of 
magnitude relative to $\vert S_t\vert$, with $\tau=0.01$ yielding roughly a 
90\% reduction in storage space without any significant loss of performance.

\section{Experiments}
\label{sec:results}

\subsection{Main results}
\label{subsec:main_results}

We evaluate \methodname on CIFAR-100 (Table~\ref{table:cifar100_combined}) and TinyImageNet (Table~\ref{table:tinyimagenet_combined}) using ResNet-18 under clean labels as well as symmetric, asymmetric, and human label noise (CIFAR-100N).\footnote{All results in the main body use the Basic \methodname without thresholding, unless stated otherwise.}  
Comparisons include standard cross-entropy (CE) with early stopping, SAT, and the KF ensemble.  
We specifically include SAT because, among self-distillation methods in the literature, it is the best performing approach that explicitly targets noisy-label scenarios. 
As DLB does not report matching numerical results, we defer comparison to Section~\ref{subsec:dlb}.
Across all settings, \methodname consistently surpasses these baselines while training a single model.

\begin{table*}[t!]
\centering
\caption{\textbf{Cross-method comparison on CIFAR-100 (Full Metrics).} Rows include only methods reporting results for all three architectures. For all metrics - Top-1, Top-5, and AURC -lower is better. Bold denotes the best value per column.}
\label{tab:cifar100_full_metrics}
\scriptsize
\begin{tabular*}{\textwidth}
{@{\extracolsep{\fill}}l|lll|lll|lll@{}}
\toprule
\multirow{3}{*}{\shortstack{\textbf{Method}\\ \textbf{(best across sources)}}} &
\multicolumn{3}{c|}{\textbf{ResNet-18}} &
\multicolumn{3}{c|}{\textbf{DenseNet-121}} &
\multicolumn{3}{c}{\textbf{ResNet-101}} \\
& \textbf{Top-1} & \textbf{Top-5} & \textbf{AURC} & \textbf{Top-1} & \textbf{Top-5} & \textbf{AURC} & \textbf{Top-1} & \textbf{Top-5} & \textbf{AURC} \\
& \textbf{Err (\%)} & \textbf{Err (\%)} & \textbf{($\times10^3$)} & \textbf{Err (\%)} & \textbf{Err (\%)} & \textbf{($\times10^3$)} & \textbf{Err (\%)} & \textbf{Err (\%)} & \textbf{($\times10^3$)} \\
\midrule
Vanilla$^{(\text{TSD})}$      & 23.76\se{.07} & 6.90 & 67.65 & 20.05 & 4.99 & 52.21 & 20.75 & 5.28 & 55.45 \\
TF\text{-}KD$^{(\text{Zipf/PS-KD})}$ & 22.71\se{.15} & 6.01 & 61.77 & 19.88 & 5.10 & 69.23 & 20.10 & 5.10 & 58.80 \\
CS\text{-}KD$^{(\text{PS-KD})}$ & 21.30 & 5.70 & 56.56 & 20.47 & 6.21 & 73.37 & 20.76 & 5.62 & 64.44 \\
LS$^{(\text{PS-KD})}$         & 20.94 & 6.02 & 57.74 & 19.80 & 5.46 & 91.06 & 19.84 & 5.07 & 95.76 \\
PS\text{-}KD$^{(\text{PS-KD/TSD})}$ & 20.82\se{.23} & 5.10 & 52.10 & 18.73 & \largebold{3.90} & 45.55 & 19.43 & 4.30 & 49.01 \\
\methodname$^{(\text{ours})}$ & \largebold{19.75}\se{.05} & \largebold{4.25}\se{.03} & \largebold{50.76}\se{.02} & \largebold{18.67}\se{.08} & 3.95\se{.03} & \largebold{44.33}\se{.88} & \largebold{18.04}\se{.10} & \largebold{3.65}\se{.07} & \largebold{42.81}\se{.34} \\
\bottomrule
\end{tabular*}
\end{table*}

\subsection{Comprehensive method-wise comparison}
\label{subsec:pskd_aux}




We evaluate \methodname against state-of-the-art CIFAR-100 (clean) results, organized into two complementary views. To ensure a fair and direct comparison, all competitor metrics presented below were copied verbatim from external papers that previously aggregated these baselines. These include numbers from PS-KD \citep{kim2020pskd}, TSD and related self-distillation methods \citep{liu2024tsd}, and label-smoothing variants such as DGD \citep{zipf_table2_proxy}; no additional local re-training or re-tuning was performed on our side.

Table~\ref{tab:cifar100_full_metrics} provides a multi-architecture analysis across three standard backbones: ResNet-18, ResNet-101, and DenseNet-121, reporting \emph{Top-1 error}, \emph{Top-5 error}, and \emph{AURC} (an indicator of confidence quality \citep{geifman2018bias}). For broader context, Table~\ref{tab:resnet18_summary} ranks a comprehensive list of baselines solely by their ResNet-18 \emph{Top-1 error}. Across all metrics and architectures, \methodname demonstrates a consistent and marked advantage over existing methods.

\paragraph{Augmentation compatibility.}
\methodname retains its accuracy gains under modern augmentations such as CutMix; 
see \App~\ref{app:cutmix} for details.

\subsection{Benchmarking in Noise Regimes}
\label{subsec:dlb}


\begin{wraptable}{r}{0.55\columnwidth}
\vspace{-0.5\baselineskip} 
  \scriptsize
\centering
    \captionof{table}{\textbf{CIFAR-100 Top-1 Error.} Summary of all baseline methods ranked by Top-1 Error. }
    \label{tab:resnet18_summary}
    \begin{tabular}{ll}
      \toprule
      \textbf{Method} & \textbf{Top-1 Err (\%)} \\
      \midrule
      Vanilla$^{(\text{TSD})}$          & 23.76\se{.07} \\
      DGD$^{(\text{Zipf})}$             & 23.52\se{.13} \\
      DDGSD$^{(\text{TSD})}$            & 23.39\se{.47} \\
      BAN$^{(\text{Zipf})}$             & 23.04\se{.04} \\
      SD\text{-}Dropout$^{(\text{TSD})}$ & 23.00\se{.00} \\
      TF\text{-}KD$^{(\text{Zipf/PS-KD})}$ & 22.71\se{.15} \\
      Zipf’s LS$^{(\text{TSD})}$        & 22.62\se{.32} \\
      FRSKD$^{(\text{TSD})}$            & 22.29\se{.14} \\
      BYOT$^{(\text{TSD/Zipf})}$        & 22.12\se{.19} \\
      TSD$^{(\text{TSD})}$              & 21.46\se{.35} \\
      DKS $^{(\text{TSD})}$              & 21.36\se{.25} \\
      CS\text{-}KD$^{(\text{PS-KD})}$     & 21.30 \\
      LS$^{(\text{PS-KD})}$              & 20.94 \\
      PS\text{-}KD$^{(\text{PS-KD/TSD})}$  & 20.82\se{.23} \\
      \textbf{VISTA}$^{(\text{ours})}$     & \textbf{19.75}\se{.05} \\
      \bottomrule
    \end{tabular}
\end{wraptable}

While Section~\ref{subsec:pskd_aux} reports standard metrics on clean-label datasets, this section evaluates \methodname\ against methods tailored to noisy-label learning, in particular DLB~\citep{shen2024dlb}. Since DLB reports results only as visual accuracy gains over a baseline on custom backbones, we estimate these gains from the plots and compare them to the relative improvements achieved by \methodname\ over our baseline on matched architecture families. This gain-vs-gain comparison enables a fair evaluation across similar model scales while avoiding confounding differences in training setups.

As seen in Table~\ref{tab:dlb_comparison_gains}, when considering all noise levels, \methodname delivers accuracy gains of up to \(+\!18.9\%\) for VGG-19 over DLB, 
with a mean advantage of about \(+\!6.5\) percentage points per architecture family and a win rate of approximately \(88\%\) across all individual comparisons, thus demonstrating a clear and substantial benefit across the noise scale.

\begin{table}[ht]
\centering
\caption{Comparison to DLB: CIFAR-100 with symmetric label noise, relative accuracy gains. We compare DLB and \methodname on matching architecture families for 4 family pairs: ResNet, VGG-16, VGG-19, and DenseNet. Bold highlights the best performer within each architecture group.
}
\label{tab:dlb_comparison_gains}
\resizebox{\columnwidth}{!}{
\begin{tabular}{lllllll}
\toprule
 & \multicolumn{6}{c}{\textbf{Symmetric Noise Level}} \\
\cmidrule(lr){2-7}
\textbf{Architecture (Method)} & \textbf{10\%} & \textbf{20\%} & \textbf{30\%} & \textbf{40\%} & \textbf{50\%} & \textbf{60\%} \\
\midrule
VGG-16 (DLB)    & \textbf{3.10} & 3.30 & 3.10 & 3.00 & 4.20 & 5.80 \\
VGG-16 (\methodname)   & 2.83 \se{.15} & \textbf{4.02} \se{.27} & \textbf{6.44} \se{.49} & \textbf{12.55} \se{1.17} & \textbf{22.29} \se{.71} & \textbf{13.34} \se{2.10} \\
\midrule
VGG-19 (DLB)    & \textbf{3.20} & 3.60 & 4.05 & 5.70 & 5.80 & \textbf{5.50} \\
VGG-19 (\methodname)   & 2.18 \se{.43} & \textbf{4.52} \se{.70} & \textbf{8.90} \se{1.33} & \textbf{16.59} \se{.96} & \textbf{24.71} \se{2.73} & 3.18 \se{1.22} \\
\midrule
ResNet32 (DLB)  & 2.50 & 2.80 & 2.60 & 2.30 & 3.70 & 2.40 \\
ResNet18 (\methodname) & \textbf{6.16} \se{.41} & \textbf{9.11} \se{.11} & \textbf{10.84} \se{.30} & \textbf{10.38} \se{.39} & \textbf{10.57} \se{.28} & \textbf{8.50} \se{.55} \\
\midrule
DenseNet40 (DLB)   & 1.80 & 1.40 & 1.00 & 1.40 & 1.50 & 3.65 \\
DenseNet121 (\methodname) & \textbf{5.87} \se{.56} & \textbf{7.08} \se{.51} & \textbf{7.31} \se{.13} & \textbf{10.43} \se{.37} & \textbf{10.01} \se{.14} & \textbf{8.82} \se{.47} \\
\midrule
WRN (DLB)       & 1.40 & 2.10 & 2.00 & 1.50 & 1.90 & 3.70 \\
ResNet110 (DLB) & 3.10 & 3.10 & 3.90 & 2.70 & 4.25 & 5.30 \\
\bottomrule
\end{tabular}
}
\end{table}



\subsection{Trajectory and Structural Consistency Analysis}

Having established that our method generalizes well across diverse settings, we next examine its training dynamics through the lens of optimization stability. Figure~\ref{fig:forget_ce_vs_cafe}b shows that \methodname substantially improves structural consistency, as evidenced by a pronounced reduction in the average $\mathrm{Deviation\ Gap}$ from 10.60\% to 2.97\%. By anchoring the model to high-fidelity latent regions learned earlier along the training trajectory, the characteristic double-ascent behavior, which is often associated with localized optimization instability, is eliminated.

\paragraph{Expert Anchor Lifespan}
Figure~\ref{fig:life_expectancy_heatmap} visualizes the duration for which each past model state remains uniquely useful as an expert anchor during training. Read each vertical bar as the \emph{lifespan} of a checkpoint: its height counts how many future epochs still credit that specific state with nonzero \emph{marginal validation coverage}. For example, if a bar of height 30 sits above epoch $x{=}48$, it indicates that the model at epoch 48 retains unique credit until epoch 78. At that point, later models fully inherit its coverage and its marginal contribution falls to zero. Color intensity encodes the magnitude of this marginal coverage; note that typically shading along the bar fades with height, which  reflects the intuitive result that an expert anchor's uniqueness is expected to decay as the optimization trajectory progresses.

\paragraph{Forward Knowledge Transfer}
When comparing the top panel of Figure~\ref{fig:life_expectancy_heatmap} representing Vanilla training and the bottom panel representing \methodname, the contrast in optimization stability is stark. Vanilla training shows many very tall, slow-fading bars, which indicates that later checkpoints repeatedly fail to inherit the unique mastery of their predecessors. This lack of inheritance leads to a diminished \textbf{structural residual}, as the model fails to retain specific samples it once correctly identified. In contrast, the bottom panel for \methodname shows much shorter, fast-fading bars, indicating a rapid handoff of useful patterns along the trajectory. 

\begin{figure}[h!]
\centering
\includegraphics[width=\linewidth]{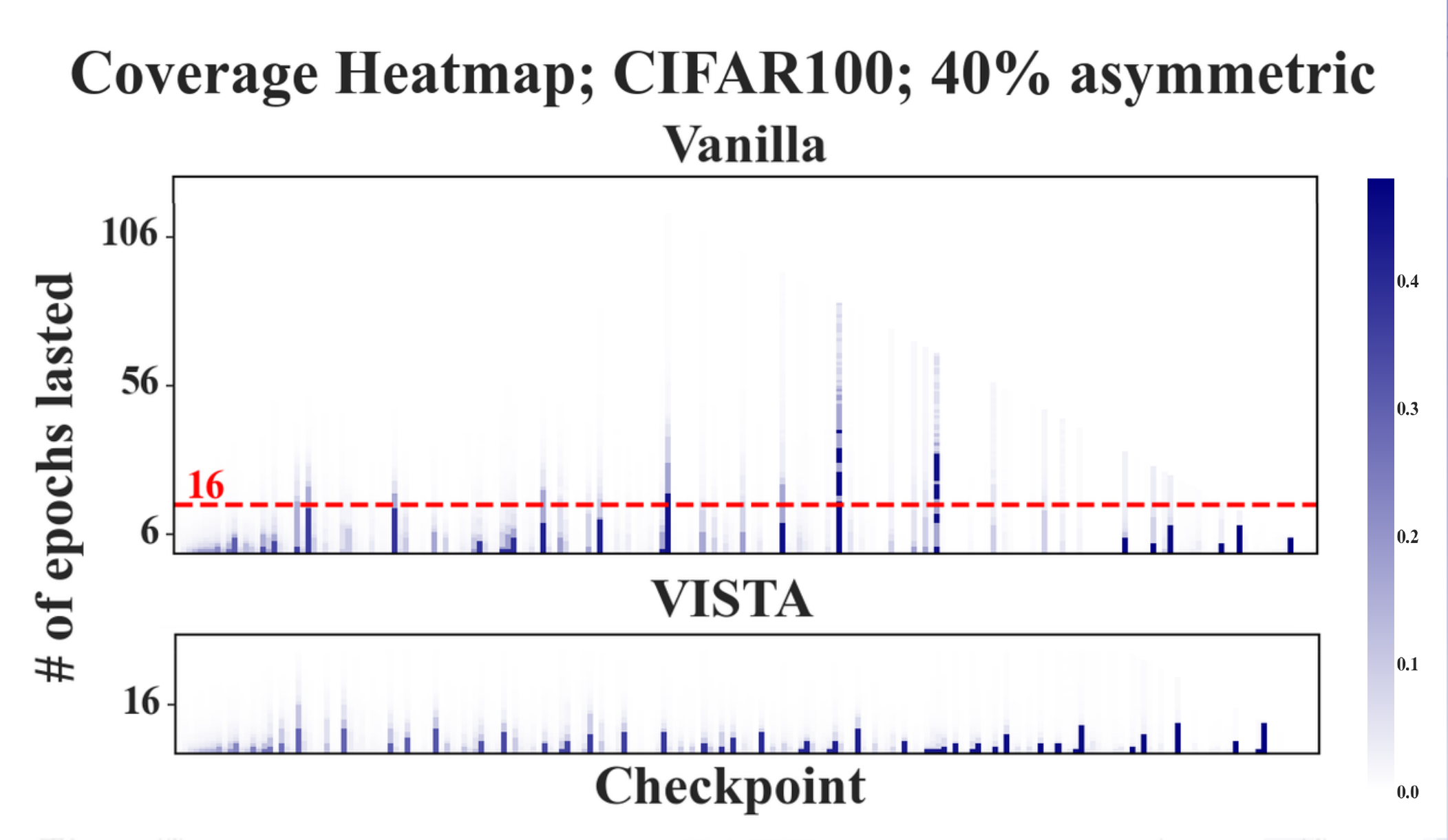}

\caption{Checkpoint life expectancy as a visualization of trajectory consistency. Each vertical bar represents the duration for which a specific model state remains uniquely useful as an expert anchor. Vanilla training exhibits long-lasting uniqueness that later states fail to absorb, while \methodname shows much shorter lifespans, which indicates a rapid and effective transfer of structural information to subsequent checkpoints.}
\label{fig:life_expectancy_heatmap}
\end{figure}

\paragraph{End-of-training (EoT) marginal coverage}
A complementary view is shown in Figure~\ref{fig:cov_at_eot}, which plots, for each epoch, its \emph{marginal coverage measured at the end of training}. The yellow curve (Vanilla) reveals several large spikes originating from mid-training checkpoints: even at the end of training, substantial portions of the validation examples remain \emph{uniquely} attributed to those earlier epochs, indicating that later models never fully inherit their coverage. In contrast, the purple curve (\methodname) concentrates almost all its mass near the final epochs, indicating that knowledge is systematically rolled forward into the final model. 

\begin{figure}[h!]
\centering
\includegraphics[width=\linewidth]{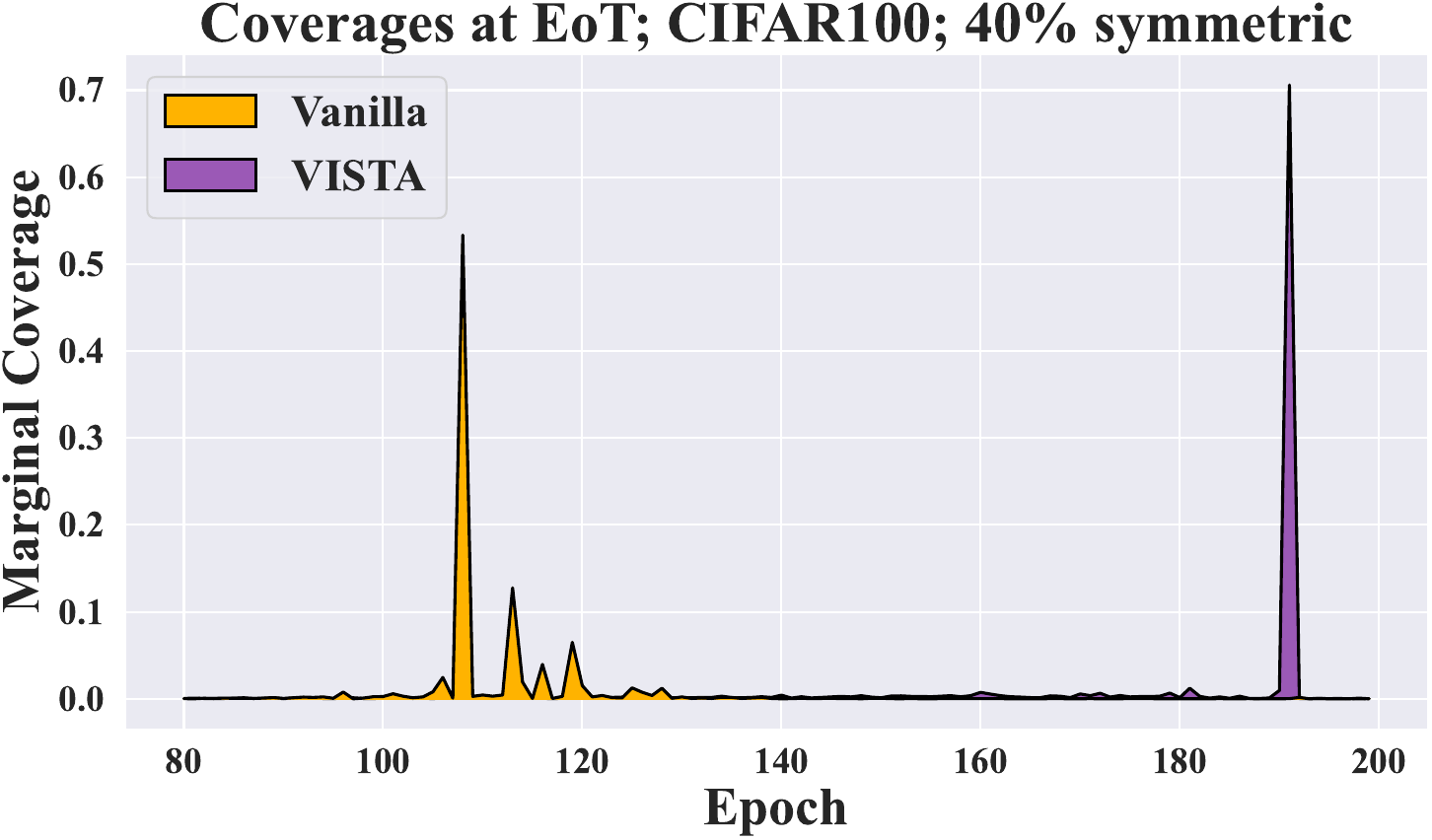}
\caption{Marginal coverage at the end of training for CIFAR-100 with 40\% symmetric noise. The yellow curve for Vanilla reveals that substantial validation coverage remains uniquely attributed to mid-training epochs. In contrast, the purple curve for \methodname concentrates mass near the final epochs, which indicates that knowledge is systematically rolled forward and embedded into the final model. 
}
\label{fig:cov_at_eot}
\end{figure}

To address the possibility that the strongest performance occurs mid-training, the quantitative comparisons reported in Section~\ref{subsec:main_results} evaluate \emph{Vanilla} under an oracle early-stopping protocol. This mechanism is standard in corrupted label training to prevent models from learning noisy patterns toward the end of the optimization path, which explains why a mid-training accuracy spike is often expected. This oracle selects the epoch with the highest validation accuracy, which corresponds precisely to the largest end of training spike in Figure~\ref{fig:cov_at_eot}. Despite this favorable baseline, \methodname achieves higher test accuracy and a significantly larger structural residual by explicitly rolling that knowledge forward to later checkpoints. This results in better generalization than freezing the model at its best mid-training state.


\subsection{Empirical Storage Efficiency}
\label{sec:storage_empirical}

The light version of \methodname, introduced in Section~\ref{subsec:light-cafe}, is designed to be substantially more efficient than the basic variant described in Section~\ref{cafe-basic}, as will be confirmed by the complexity analysis in Section~\ref{sec:complex-asys}. To assess the trade-off in practice, we compare the two methods empirically. 

Table~\ref{tab:fat_vs_thresh} shows that the accuracy of the two methods remains 
virtually identical when the tolerance parameter of Light \methodname is set to 
\(\tau=0.01\). Further evidence is provided in Figure~\ref{fig:n_ckpts_per_epoch},
which tracks the size of the contributing set \(\mathcal{A}_{t}\) 
(see \App~\ref{app:light-cafe-alg}, Alg.~\ref{alg:fat_coverage_distill-generalized}) throughout training.  
The largest value observed along each curve essentially represents the 
amount of checkpoint storage required to execute the entire run end-to-end. With \(\tau=0.01\), these maxima drop from 173–176 checkpoints down to just 7–12 
across the three different noise settings.
Despite these drastic reductions, accuracy remains essentially unchanged
(Table~\ref{tab:fat_vs_thresh}), confirming that \methodname's advantage stems 
from the quality of the historical signals that are being retained, not from maintaining 
a large archive.

\begin{figure}[h]
    \centering
    \includegraphics[width=1.0\columnwidth]{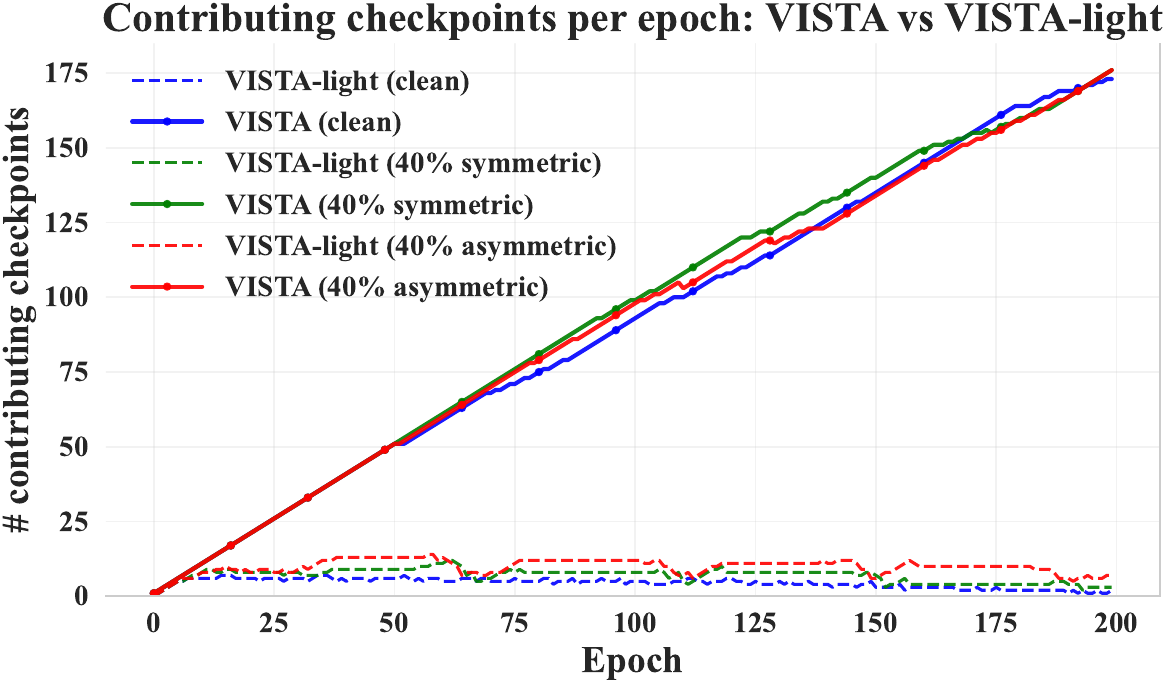}
    \vspace{-15pt}
    \caption{Pruning of the contributing set $\mathcal{A}_t$ on CIFAR-100. \methodname-light ($\tau{=}0.01$, dashed) maintains a near-constant storage footprint compared to the linear growth of Basic \methodname (solid).}
    \label{fig:n_ckpts_per_epoch}
    \vspace{-10pt}
\end{figure}

\begin{table*}[t]
\centering
\scriptsize
\caption{\textbf{Accuracy and maximum contributing checkpoints for Basic and Light \methodname.}
Tolerance of \(\tau{=}0.01\) prunes low-coverage checkpoints, slashing storage needs while preserving accuracy.}
\label{tab:fat_vs_thresh}
\begin{tabular}{l|cc|cc|cc}
\toprule
\multirow{2}{*}{\textbf{Method}} 
 & \multicolumn{2}{c|}{\textbf{Clean}} 
 & \multicolumn{2}{c|}{\textbf{40\% Symmetric}} 
 & \multicolumn{2}{c}{\textbf{40\% Asymmetric}} \\
 & Acc. (\%) & Required ckpt space (\#)
 & Acc. (\%) & Required ckpt space (\#)
 & Acc. (\%) & Required ckpt space (\#) \\
\midrule
\emph{\methodname} & 80.25 \se{.05} & 173 & 69.75 \se{.14} & 176 & 67.17 \se{.10} & 176 \\
\emph{Light \methodname} (\(\tau{=}0.01\)) & 80.08 \se{.15} & 7 & 69.55 \se{.15} & 12 & 66.88 \se{.21} & 12 \\
\bottomrule
\end{tabular}
\end{table*}

\subsection{Ablation summary}

Due to space constraints, the full ablation studies and their empirical evidence 
have been moved to \App~\ref{app:ablation}. Below we provide only the 
headline findings from all ablations:

\begin{itemize}[nosep, leftmargin=*]

  \item \textbf{EMA variants.} \emph{Parameter-space} EMA variants provide no extra benefit, coverage in \emph{prediction space} matters most (see \App~\ref{subsec:cafe_ema_ablation}).
  
  \item \textbf{Coverage vs. raw accuracy.} Simply relying on \emph{raw validation accuracy} to guide training is insufficient; the key contributor to performance is constructing teacher weights based on \emph{marginal coverage} (see \App~\ref{subsec:ablate_accuracy_only}).
  
  \item \textbf{Checkpoint order: accuracy vs. chronology.} Sweeping through past checkpoints in reverse chronological order (instead of sorting by validation accuracy) yields worse results, underscoring the importance of using validation accuracy to rank historical models (see \App~\ref{subsec:ablate_chronological}).

   \item \textbf{Temporal smoothing.} Smoothing across adjacent checkpoints before weighting offers no gain; validation-guided coverage already ensures stable soft labels (see \App~\ref{subsec:ablate_smoothing3}).
  

      \item \textbf{Blending schedule (\texorpdfstring{$\beta$}{beta}).} 
  Fixed mixtures (\(\beta = 0\) or \(\beta = 1\)) severely under-perform, showing that neither ignoring the teacher nor relying on it exclusively is effective.  
  In contrast, any smoothly increasing schedule (linear, cosine, exponential with different \(k\)) yields nearly identical accuracy, indicating that the precise functional form is not important (see \App~\ref{appendix:beta_schedule_ablation}).

\end{itemize}


\subsection{Discussion of Empirical Findings} 
These results collectively demonstrate that \methodname successfully transforms lingering uniqueness into transferred knowledge. By treating high-fidelity states as expert anchors, the framework ensures that earlier mastery is systematically absorbed into the final converged model, thereby maximizing the \textbf{structural residual} and minimizing the \textbf{deviation gap}. This targeted preservation of knowledge explains the measurable superiority of \methodname across diverse benchmarks, providing better accuracy while maintaining single-model inference efficiency.

While we evaluate \methodname on standard architectures , Trajectory Deviation also occurs in heavier models like Vision Transformers. The Knowledge Fusion (KF) framework proved both the existence of this phenomenon in such models and the efficacy of validation-driven signals in mitigating it. Since \methodname harnesses this same concept, transitioning from a post-hoc ensemble to an online, single-model approach , we expect results to generalize to these architectures, a path hindered here by resource constraints. However, our evaluation across diverse datasets and noise regimes provides robust evidence for these claims.

\section{Theoretical \& complexity analysis}
\label{sec:asys}

\paragraph{Complexity} \label{sec:complex-asys}
The computational cost of Basic \methodname is 
$\mathcal{O}(n \log n + n m)$ in time and $\mathcal{O}(n)$ in space, 
where $n=|S_t|$ is the number of stored checkpoints and $m$ is the validation set size.  
For Light \methodname, the cost depends on the size of the contributing set 
$\mathcal{A}_t \subseteq S_t$, which shrinks substantially for even small 
values of $\tau$. A full step-by-step complexity derivation is provided 
in \App~\ref{app:complexity}.

\paragraph{Theoretical Guarantee} We prove that once a checkpoint’s marginal coverage becomes zero at some epoch, 
it can never regain positive marginal coverage at later epochs.  
Consequently, Light \methodname with $\tau=0$ is provably equivalent to Basic \methodname while achieving a significant reduction in storage. Larger reduction is obtained when $\tau>0$. The formal claim and complete proof are given in \App~\ref{app:proof_tau_equivalence}.

\section{Summary and discussion}

We presented \methodname, an online, validation-aware self-distillation approach that mitigates trajectory deviation in deep neural networks. By weighting historical predictions by their unique contribution to the structural residual, \methodname improves both robustness and generalization, while preserving single-model inference efficiency. In contrast to existing methods, it explicitly selects which past predictions to trust using validation coverage. A promising direction for future work is to identify patterns and examples most susceptible to trajectory deviation and to develop preventative strategies, rather than to focus exclusively on recovering mastered knowledge after it has been lost.

We additionally introduced Light \methodname, which provides a 'knob' to control the trade-off between complexity and performance; remarkably, even the most lenient setting reduces resource overhead by approximately 90\% while maintaining the same generalization and consistency gains.

Our findings support two claims: (1) 
Trajectory deviation is a pervasive and 
functionally significant effect in deep learning, representing a structural optimization failure that occurs even in the absence of label noise.
(2) Targeted mitigation via validation-aware weighting of historical knowledge yields measurable improvements in robustness and accuracy. 


\subsection*{\textbf{Acknowledgments}}
This work was supported by a grant from the Gatsby Charitable Foundation and AFOSR award FA8655-24-1-7006.

\bibliography{vista}
\bibliographystyle{icml2026}

\newpage
\appendix
\onecolumn
\section{Light \methodname pseudocode}
\label{app:light-cafe-alg}

For completeness, we provide the full Light \methodname\ procedure introduced in Section~\ref{subsec:light-cafe}.

\begin{algorithm}[H]
\caption{Light \methodname: epoch $t$} 
\label{alg:fat_coverage_distill-generalized}

\textbf{Input}:  Validation set $\mathcal{V}$; contributing set $\mathcal{A}_{t-1}$; coverage subsets $\{C_s\}_{s\in\mathcal{A}_{t-1}\cup\{t\}}$; validation accuracies $\{a_s\}_{s\in\mathcal{A}_{t-1}\cup\{t\}}$; checkpoint predictions $\{p_s(x)\}_{s\in\mathcal{A}_{t-1}\cup\{t\}}$; scheduler value $\beta_t \in [0,1]$; tolerance parameter $\tau\geq 0$.

\textbf{Output}: blended target for next epoch optimization.

\begin{algorithmic}[1]
\State \textbf{Order checkpoints:} $\mathcal{K} \gets \{\mathcal{A}_{t-1}\cup\{t\}\}$ sorted by $a_s,~s\in \mathcal{A}_{t-1}\cup\{t\}$ (high $\to$ low) 
\label{light-step2}
\State \textbf{Marginal coverage sweep:} $U \gets \emptyset$; \textbf{for} $s \in \mathcal{K}$ \textbf{do} $\Delta_s \gets |C_s \setminus U|$; \; $U \gets U \cup C_s$ 
\State \textbf{Update contributing set:}  $\mathcal{A}_t \gets \{s\in \mathcal{K}, \Delta_s\ > \tau\}$
\State \textbf{Normalize coverage gains:} $\widehat{\Delta}_s \gets \Delta_s \Big/ \sum_{j \in \mathcal{A}_t} \Delta_j$ 
\label{light-step5}
\State \textbf{Construct coverage-weighted teacher:} $Q_t(x) \gets \sum_{s \in \mathcal{A}_t} \widehat{\Delta}_s \, p_s(x)$
\label{light-step6}
\State \textbf{Form blended target for next epoch:} $\tilde y_{t+1}(x) \gets (1-\beta_t)\,\text{one-hot}(y) + \beta_t\,Q_t(x)$ 
\State \textbf{Return}  $\tilde y_{t+1}(x)$
\end{algorithmic}
\end{algorithm}

\section{Compatibility with data augmentation}
\label{app:cutmix}
To verify compatibility with common augmentations, we tested \methodname using 
CutMix, a widely adopted augmentation for image classification, with a ResNet-18 
backbone. \methodname maintains its accuracy gains under CutMix, confirming that 
its coverage-aware distillation integrates smoothly with representative, state-of-the-art augmentation strategies (see Table~\ref{tab:cutmix_compact}).

\begin{table}[ht]
\centering
\scriptsize
\caption{CutMix. 
Mean test accuracy with and without CutMix augmentations on CIFAR100.
}
\label{tab:cutmix_compact}
\begin{tabular}{l|c|c|c|c|c|c}
\toprule
\multirow{2}{*}{Method} & \multicolumn{3}{c|}{CIFAR-100 Symmetric} & \multicolumn{2}{c|}{CIFAR-100 Asym.} & TinyImageNet \\
 & Clean & 20\% & 40\% & 20\% & 40\% & Clean \\
\midrule
Vanilla + Early Stopping        & 78.66\se{.11} & 64.98\se{.11} & 59.37\se{.36} & 66.61\se{.12} & 49.57\se{.10} & 65.37\se{.01} \\
\hspace{1em} + CutMix     & \textbf{80.75}\se{.16} & 73.60\se{.15} & 67.15\se{.46} & 75.10\se{.21} & 62.81\se{.31} & 68.45\se{.19} \\
\midrule
\methodname          & 80.25\se{.05} & 74.09\se{.01} & \textbf{69.75}\se{.14} & 75.61\se{.18} & 67.17\se{.10} & 68.91\se{.01} \\
\hspace{1em} + CutMix       & 
\textbf{80.74}\se{.09} & \textbf{74.45} \se{.19} & 69.19\se{.30} & \textbf{75.95}\se{.12} & \textbf{71.47}\se{.16} & \textbf{70.62}\se{.04} \\
\bottomrule
\end{tabular}
\end{table}


\section{Ablation study: full results}
\label{app:ablation}

\subsection{\methodname-inspired EMA} 
\label{subsec:cafe_ema_ablation}
Classical Exponential Moving Average (EMA) smooths parameter updates by maintaining a running average of the model weights across epochs. In order to align this idea with our coverage notion, we devised a variant of EMA that takes advantage of the scores used by \methodname, replacing mere temporal averaging by reweighting past checkpoints according to their \emph{marginal validation coverage}. Specifically, at the end of each epoch we compute the incremental validation coverage contributed by each earlier checkpoint, normalize these increments into weights, and construct a teacher model in parameter space by averaging the corresponding checkpoints. The current model is then blended towards this teacher, yielding the method termed \methodname-EMA. 

Table~\ref{tab:cafe_ema_ablation} summarizes these results for CIFAR-100 (clean). Both Vanilla EMA and \methodname-EMA ($2^{nd}$ and $4^{th}$ rows) use the same Cross Entropy loss as in Vanilla ($1^{st}$ row); the only difference lies in the dynamic adjustment of network weights via EMA. Vanilla EMA (decay = 0.999) performs slightly worse than Vanilla ($78.32 \pm 0.14$ vs. $78.66 \pm 0.11$), and the \methodname-EMA variant underperforms further at $76.60 \pm 0.23$. In contrast, Basic \methodname achieves $80.25 \pm 0.05$. This suggests that while marginal coverage weighting is effective for soft label formation in self-distillation, applying it directly in parameter-space (i.e., averaging and blending the network weights themselves across checkpoints, rather than averaging predictions or soft labels) does not yield the same benefits, and may even hinder optimization. This confirms that validation-coverage weighting is most effective when applied in the \emph{prediction space}, rather than in the parameter space.

\begin{table*}[t]
\centering
\begin{minipage}{0.44\textwidth}
\centering
\vspace{-0.23cm}
\scriptsize
\caption{\methodname-inspired EMA ablation on CIFAR-100 (clean).}
\label{tab:cafe_ema_ablation}
\begin{tabular}{l|l}
\toprule
Method & CIFAR-100 \\
\midrule
Vanilla + Early Stopping (ES) & $78.66 \semth{.11}$ \\
Vanilla EMA + ES \tiny{(decay = 0.999)} & $78.32 \semth{.14}$ \\
\midrule
\methodname & $\mathbf{80.25} \semth{.05}$ \\
\methodname-EMA & $76.60 \semth{.23}$ \\
\bottomrule
\end{tabular}
\end{minipage}
\hspace{0.02\textwidth}
\begin{minipage}{0.52\textwidth}
\centering
\scriptsize
\caption{Joint ablation results on CIFAR-100 with 40\% label noise. 
}
\label{tab:cafe_ablate_noise}
\begin{tabular}{l|cc}
\toprule
Method & Sym 40\% & Asym 40\% \\
\midrule
Vanilla + Early Stopping & $59.37 \semth{.36}$ & $49.57 \semth{.10}$ \\
\midrule
Accuracy-only weighting (\ref{subsec:ablate_accuracy_only}) & $67.73 \semth{.16}$ & $62.37 \semth{.34}$ \\
Chronological coverage (\ref{subsec:ablate_chronological}) & $\mathbf{69.49} \semth{.31}$ & $64.18 \semth{.32}$ \\
Temporal smoothing (\ref{subsec:ablate_smoothing3}) & $69.37 \semth{.18}$ & $64.50 \semth{1.30}$ \\
\methodname & $\mathbf{69.75} \semth{.14}$ & $\mathbf{67.17} \semth{.10}$ \\
\bottomrule
\end{tabular}
\end{minipage}
\end{table*}

\subsection{Coverage vs. accuracy-only averaging}
\label{subsec:ablate_accuracy_only}

We now consider a variant that \emph{ignores coverage entirely} and uses the \emph{normalized validation accuracy} of each checkpoint as a direct weight in the teacher, without considering marginal coverage. The key question is whether \methodname's advantage arises from its \emph{coverage-aware credit} or simply from \emph{favoring high-accuracy checkpoints}. Results are reported in Table~\ref{tab:cafe_ablate_noise}, showing that \methodname surpasses the accuracy-only variant under both types of noise, further indicating that \textbf{coverage-aware credit assignment is crucial}.

\subsection{Chronological coverage vs. accuracy-sorted marginal coverage}
\label{subsec:ablate_chronological}

We investigate an alternative choice for \methodname where we retain the coverage notion 
but replace Step~\ref{step1} in Algorithm~\ref{alg:fat_coverage_distill} with a simple chronological ordering: 
checkpoints are sorted by time so that later checkpoints come first when assigning credit.
This contrasts with \methodname, which sorts checkpoints by validation accuracy and uses marginal coverage in that order. Results are reported in Table~\ref{tab:cafe_ablate_noise}, showing that \methodname is superior or on par for both noise  settings, 
validating our decision to base marginal coverage on accuracy-sorted checkpoints rather than chronological order.

\subsection{Temporal smoothing of soft labels}
\label{subsec:ablate_smoothing3}

The current ablation applies a smoothing window (of size 3) across adjacent checkpoints before forming the teacher, aiming to reduce noisy per-epoch fluctuations. The question is whether \methodname's validation-guided ordering and marginal coverage already provide sufficient stability.  Results are reported in Table~\ref{tab:cafe_ablate_noise}, showing that \methodname remains superior across both noise types, suggesting that validation-guided marginal coverage already yields robust soft labels, with no extra benefit from local temporal smoothing.

\subsection{Effect of the \texorpdfstring{$\beta$}{beta} schedule}
\label{appendix:beta_schedule_ablation}

We now study the effect of the $\beta$ schedule used to interpolate between hard labels and the teacher distribution in \methodname. Recall that in the main method we define a normalized epoch index
\begin{equation}
    u_t \;=\; \frac{t-1}{E-1} \in [0,1],
\end{equation}
and use an exponential schedule of the form
\begin{equation}
    \beta_t^{(k)} \;=\; \frac{1 - \exp\!\bigl(-k\,u_t\bigr)}{1 - \exp(-k)},
\end{equation}
where the default choice in \methodname is \(k = 2\).  
In this ablation we vary \(k\) in this exponential family (while keeping all other hyperparameters fixed), and also compare against simple baselines: constant schedules \(\beta_t \equiv c\) with \(c \in \{0, 0.5, 1\}\), a linear schedule \(\beta_t = u_t\), and a cosine schedule \(\beta_t = \tfrac{1 - \cos(\pi u_t)}{2}\).

Table~\ref{tab:beta_schedule_ablation} reports test accuracy on CIFAR-100 with 40\% symmetric noise. 
Constant schedules provide useful reference points: 
$\text{fix-0}$ corresponds to relying entirely on hard labels with no teacher signal, 
whereas $\text{fix-1}$ corresponds to relying solely on the teacher distribution. 
Both extremes perform substantially worse than any gradually increasing schedule. 
In contrast, smooth schedules such as linear, cosine, or exponential (\(k = 1, 2, 3\)) all yield similar accuracies, 
showing that the important design choice is to let $\beta_t$ increase over training, 
and that \methodname is robust to the exact functional form of this increase.

\begin{table}[H]
    \centering
    \caption{Effect of the $\beta$ schedule on CIFAR-100 with 40\% symmetric noise. We report mean test accuracy (\%) and standard error. Smooth increasing schedules (linear, cosine, exponential with different $k$ values) behave similarly, while constant schedules perform substantially worse.}
    \label{tab:beta_schedule_ablation}
    \begin{tabular}{lc}
    \toprule
    Schedule & Test accuracy (\%) \\
    \midrule
    fix-0           & \(59.39\,\semth{.41}\) \\
    fix-0.5         & \(61.68\,\semth{.38}\) \\
    fix-1           & \(0.94\,\semth{.06}\) \\
    \hline
    linear          & \(69.26\,\semth{.08}\) \\
    cosine          & \(69.36\,\semth{.02}\) \\
    exp; \(k=1\)    & \(69.23\,\semth{.15}\) \\
    exp; \(k=2\) \text{(default)} & \(69.75\,\semth{.14}\) \\
    exp; \(k=3\)    & \(69.53\,\semth{.59}\) \\
    \bottomrule
    \end{tabular}
\end{table}

\section{Theoretical analysis}
\label{app:theo_and_comlex_asys}

\subsection{Complexity Derivation}
\label{app:complexity}
To analyze the complexity of Alg.~\ref{alg:fat_coverage_distill}, let $n = |S_t|$ and $m$ the size of the validation set. 

\textbf{Time complexity:} 
\begin{enumerate}[label=\roman*) ,leftmargin=0.8cm,topsep=0pt]
\setlength\itemsep{0.1em}
  \item Step~\ref{step1} is dominated by $\mathcal{O}( n \log n )$.
  \item Step~\ref{step2} is dominated by $\mathcal{O}( n m )$.
  \item Steps~\ref{step3}-\ref{step4} are dominated by $\mathcal{O}( n )$.
\end{enumerate}
\emph{Overall time complexity} is 
$\mathcal{O}\bigl(n \log n + n m\bigr).$

\textbf{Space complexity:} 
Storing predictions $\{p_s(x)\}$ for $n$ checkpoints requires 
$\mathcal{O}(n ).$

Similarly, the time and space complexity of Alg.~\ref{alg:fat_coverage_distill-generalized} in the appendix is dominated by the size of the set of \emph{contributing checkpoints} $\mathcal{A}_t \subseteq S_t$, where $\lvert \mathcal{A}_t \rvert\leq \lvert S_t\rvert$. The larger the tolerance parameter $\tau$ is, the larger the decrease in complexity achieved by Light \methodname as compared to Basic \methodname. Our empirical study shows that there is already a large gain even for $\tau=0.01$.

\subsection{Marginal Coverage Claim: Formal Statement and Proof}
\label{app:proof_tau_equivalence}

For any checkpoint $s$, if there exists an epoch $t$ at which its marginal coverage is null, i.e. $\Delta_s=0$, then $\Delta_s=0$ for all subsequent epochs. Consequently, the space complexity is bounded by the size of the set of contributing checkpoints for which $\Delta_s>0~\forall t$ -- new checkpoints may be added to this set, but once a  checkpoint is discarded, it cannot reenter.

This is stated more formally in the next claim:

\begin{claim}
\label{claim:1}
Let $S_t=\{s_1,\ldots,s_t\}$ denote the ordered set of checkpoints up to epoch $t$, ranked by the size of their coverage subset $C_{s_i}$ - the set of points that checkpoint $s_i$ correctly classifies in the validation set. Using (\ref{eq:weights}), define the marginal coverage of checkpoint $s_i$ at epoch $t$
\begin{equation*}
\Delta^t_{s_i} \;=\; \bigl|\, C_{s_i} \setminus \bigcup_{j=1}^{i-1} C_{s_j} \,\bigr|.
\end{equation*}
Let $\mathcal{B}_t \subseteq S_t$ denote the support set of checkpoints at epoch $t$, defined as  
\begin{equation*}
\mathcal{B}_t \;=\; \{\, s_i \in S_t \;\mid\; \Delta^t_{s_i}  > 0 \,\}.
\end{equation*}

Given an epoch $t$ and some future epoch $T>t$, $~\mathcal{B}_T\cap \{S_t\setminus \mathcal{B}_t\} = \emptyset$. 
\end{claim}
\begin{proof}
By contradiction. Suppose there is a checkpoint $s$ and epoch $T$ such that $s\in S_{T-1}\subset S_T$. Assume that $s\in \{S_{T-1}\setminus \mathcal{B}_{T-1}\}\implies s\notin\mathcal{B}_{T-1}$ and $s\in\mathcal{B}_T$. Consider 2 cases:
\begin{enumerate}

\item $\lvert C_{s}\rvert\geq \lvert C_{T}\rvert$: By construction checkpoint $T$ succeeds checkpoint $s$ in both ordered sets $S_{T-1}$ and $S_T$. Therefore the marginal coverage $\Delta_s$ in both epochs is the same, and $s\notin\mathcal{B}_{T-1} \implies s\notin\mathcal{B}_T$, in contradiction to the assumption. 

\item $\lvert C_{s}\rvert<\lvert C_{T}\rvert$: Let $\tilde S$ denote the set of checkpoints that precede $s$ in ordered set $S_{T-1}$. It follows that the set of checkpoints that precede $s$ in ordered set $S_{T}$ is $\tilde S\cup \{T\}$. 
\begin{eqnarray*}
s\notin\mathcal{B}_{T-1} &~~\implies~~ \Delta^{T-1}_{s} =0 &~~\implies~~ C_{s} \subseteq \bigcup_{q\in\tilde S} C_{q} \subseteq \bigcup_{q\in\tilde S\cup \{T\}} C_{q}  \\
s\in\mathcal{B}_{T} &~~\implies~~ \Delta^T_{s} \neq 0 &~~\implies~~ C_{s} \not\subseteq \bigcup_{q\in\tilde S\cup \{T\}} C_{q} 
\end{eqnarray*}
which is a contradiction.

\end{enumerate}
\end{proof}

\section{Validation accuracy as a surrogate for generalization}
\label{app:val_as_surrogate}

We next justify using validation accuracy as the weighting signal in \methodname.  
Figure~\ref{fig:acc_over_time_test_vs_val_allinone} shows standard Vanilla learning curves (without \methodname) across multiple noise distributions and learning-rate schedules.  
Although label noise can cause slight divergence between validation and test sets, their accuracies still track each other closely under all regimes.  
This consistent co-variation confirms that validation accuracy reliably reflects generalization trends, even when labels are corrupted.  
Consequently, marginal coverage weights derived from validation scores enable \methodname to identify expert anchors and propagate their structural residual forward without ever referencing the test set.

\begin{figure}[h!]
  \centering
  \includegraphics[width=\linewidth]{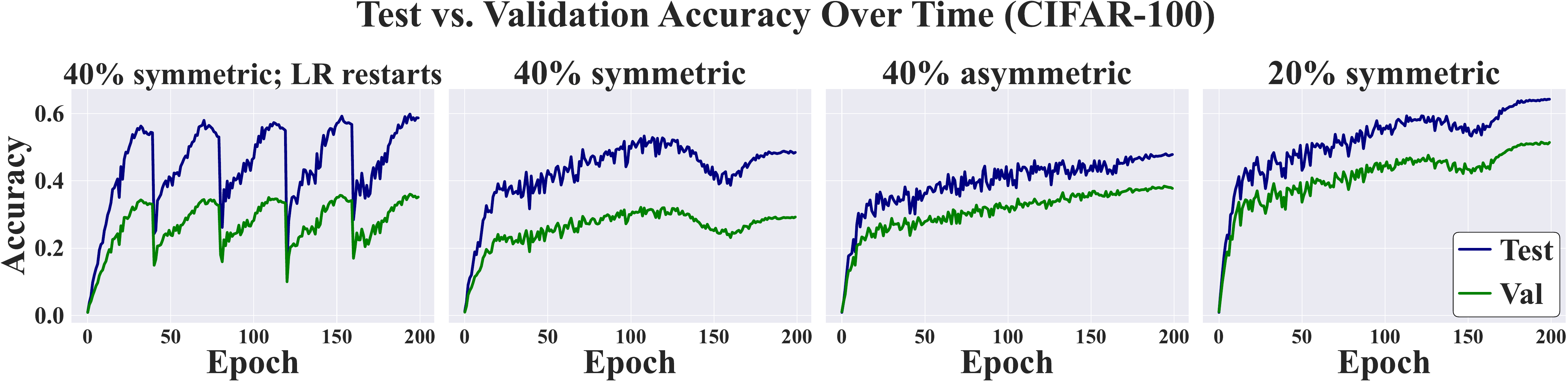}
    \caption{
        Test vs.\ validation accuracy over time for CIFAR-100 under 4 settings (from left to right):  
        1) 40\% symmetric noise with learning-rate restarts;  
        2) 40\% symmetric noise (monotonic LR);  
        3) 40\% asymmetric noise; and  
        4) 20\% symmetric noise.  
        Blue curves show test accuracy, noting that the test data lacks label noise (thus the high accuracy values). The green curves show validation accuracy, noting that validation data has the same label noise distribution as the training data.
        }
  \label{fig:acc_over_time_test_vs_val_allinone}
\end{figure}

\section{Implementation details}
\label{app:implementationdetails}

We conducted experiments on three image classification datasets: CIFAR-100 \citep{krizhevsky2009learning}, TinyImageNet \citep{le2015tiny}, and CIFAR-100N \cite{wei2022learning}.

On CIFAR-100 and TinyImageNet without label noise, all models were trained for 200 epochs with batch size 32, learning rate 0.01, SGD (momentum 0.9, weight decay 5e-4), cosine annealing, and standard augmentations (horizontal flip, random crop). For noisy label experiments, we used cosine annealing with warm restarts (every 40 epochs), a 0.1 learning rate updated per batch, and a batch size of 64.

To evaluate the impact of blending soft and hard labels, we conducted experiments using a variety of \(\beta\) schedules - both constant and varying - to govern the final mixture,  
\[
\tilde y_{t+1}(x) = (1-\beta_t)\,\mathrm{one\text{-}hot}(y) + \beta_t\,Q_t(x).
\]  
To avoid tuning \(\beta\) specifically for optimal performance, we first tested candidate functions on a small subset of the data and selected the most promising configuration, which was then applied to the full evaluation as reported in Section~\ref{sec:results}.

For a fair comparison, we re-trained the SAT method~\citep{huang2020sat} shown in Tables~\ref{table:cifar100_combined} and~\ref{table:tinyimagenet_combined} using the same architecture, data, and training scheme as our approach, following the additional guidelines in the paper. As this method is designed to integrate with existing training schemes, this setup ensures consistency. All experiments were run on A5000 GPUs.

\section{Injecting label noise} 
\label{app:injectinglabelnoise}
For label noise experiments, we injected noise using two standard methods \citep{patrini2017making}:
\begin{enumerate}
\item \textbf{Symmetric noise:} a fraction $p$ of labels is randomly selected and replaced uniformly with a different label.
\item \textbf{Asymmetric noise:} a fraction $p$ of labels is randomly selected and altered using a fixed label permutation.
\end{enumerate}

\end{document}